\documentclass[sigconf, nonacm]{acmart}
\renewcommand\footnotetextcopyrightpermission[1]{}
\AtBeginDocument{%
  }

\usepackage{subcaption}

\acmSubmissionID{192}

\begin{document}

\title{ARTA: Adversarial--Robust Multivariate Time--Series Anomaly Detection via Sparsity--Constrained Perturbations}

\author{Hadi Hojjati}
\affiliation{%
  \institution{McGill University \\ MILA - Quebec AI Institute}
  \city{Montreal}
  \state{Quebec}
  \country{Canada}
}
\email{hadi.hojjati@mcgill.ca}

\author{Narges Armanfard}
\affiliation{%
  \institution{McGill University \\ MILA - Quebec AI Institute}
  \city{Montreal}
  \state{Quebec}
  \country{Canada}
}
\email{narges.armanfard@mcgill.ca}


\begin{abstract}
Time--series anomaly detection (TSAD) is a critical component in monitoring complex systems, yet modern deep learning–based detectors are often highly sensitive to localized input corruptions and structured noise. We propose \textbf{ARTA} (\textit{\textbf{A}dversarially \textbf{R}obust multivariate \textbf{T}ime--series \textbf{A}nomaly detection via Sparsity--Constrained Perturbations}), a joint training framework that improves detector robustness through a principled min--max optimization objective.
ARTA comprises an anomaly detector and a sparsity--constrained mask generator that are trained simultaneously. The generator identifies minimal, task--relevant temporal perturbations that maximally increase the detector’s anomaly score, while the detector is optimized to remain stable under these structured perturbations. The resulting masks characterize the detector’s sensitivity to adversarial temporal corruptions and can serve as explanatory signals for the detector’s decisions. This adversarial training strategy exposes brittle decision pathways and encourages the detector to rely on distributed and stable temporal patterns rather than spurious localized artifacts.
We conduct extensive experiments on the TSB--AD benchmark, demonstrating that ARTA consistently improves anomaly detection performance across diverse datasets and exhibits significantly more graceful degradation under increasing noise levels compared to state--of--the--art baselines.
\end{abstract}





\maketitle
\section{Introduction}
Multivariate time--series anomaly detection (TSAD) is the foundation of safety--critical monitoring systems in domains such as transportation \citep{ecml} and healthcare. In these applications, anomaly scores are not merely used to rank events, but directly drive alarms, mitigation strategies, and human intervention. Accordingly, TSAD systems must exhibit score stability: small, localized perturbations in the input should not induce large changes in anomaly scores unless the underlying system dynamics have meaningfully changed. Recent evidence, however, indicates that this requirement is rarely met. Despite their architectural complexity, state--of--the--art deep TSAD models evaluated on modern benchmarks often yield marginal improvements over simple baselines while exhibiting severe fragility \citep{kim2022towards}. In real deployments, commonplace phenomena such as transient noise, sensor jitter, or partial channel failures can trigger sharp spikes in anomaly scores even under nominal operating conditions, fundamentally undermining their reliability.

This fragility stems from a fundamental misalignment between training objectives and deployment requirements \citep{wang2025improvedtimeseriesanomaly}. Most TSAD methods are optimized to minimize point--wise reconstruction or prediction errors under implicit i.i.d. train–test assumptions. Such objectives incentivize the memorization of local, high--frequency artifacts rather than the learning of a compact, noise--robust representation of normal system behavior. Consequently, ubiquitous deployment noise is frequently misinterpreted as anomalous \citep{RethinkingRobust}. The problem is further amplified by window--based, segment--level detection pipelines: while intended to capture temporal dependencies, they allow small, localized perturbations to propagate through an entire segment’s representation. This induces a strong dependence of anomaly scores on the perturbation’s position within the sliding window, leading to inconsistent and unstable detections across otherwise identical events \citep{Wang_2025}.

Although recent work explores data augmentation \citep{Luo_Cheng_Wang_Xu_Ni_Yu_Zhang_Liu_Chen_Chen_Zhang_2023}, corruption--aware training \citep{corruptionaware}, and robust losses \citep{kim2025causalityaware}, these approaches mainly improve average--case performance and do not explicitly address worst--case, localized temporal perturbations. They lack mechanisms to enforce score consistency under structured, correlated noise, which is the dominant failure mode in real deployments. Time--series corruptions span diffuse noise, correlated disturbances, and sparse local failures; while existing methods handle the first, they remain brittle to the other regimes \citep{gu2025imperceptibleadversarialattackstime}. We argue that robustness in TSAD should be induced directly through the training objective by explicitly penalizing localized temporal sensitivity. This perspective motivates approaches that embed robustness into the learning process itself, rather than treating it as post hoc additions.

Motivated by this principle, we propose ARTA, a joint training framework that explicitly targets score stability under localized perturbations. ARTA couples an anomaly detector with a sparsity--constrained temporal mask generator in an adversarial game: the generator learns to identify minimal temporal regions whose perturbation maximally disrupts the detector’s anomaly score. Crucially, sparsity constraints force these perturbations to remain compact and localized, closely mirroring realistic sensor failures, such as dropouts, spikes, saturation, or miscalibration, which are structured and temporally correlated.

Despite its conceptual simplicity, this paradigm remains largely unexplored in TSAD. Unlike prior masking--based heuristics \citep{LEEmask}, the generator is not a passive explanation module but an active component that regularizes the detector’s decision boundary during training. To the best of our knowledge, ARTA is the first framework to jointly leverage adversarial masking to explicitly regularize temporal robustness in TSAD. We evaluate ARTA on diverse datasets from the TSB--AD benchmark, utilizing metrics that assess both detection accuracy and robustness to structured corruptions. Our results demonstrate consistent improvements in anomaly detection performance alongside substantial gains in score stability under localized perturbations. Our main contributions are summarized as:
\begin{enumerate}
\item We argue that instability to localized temporal corruptions is a fundamental limitation of current paradigms and demonstrate that resolving it is a key driver for superior performance on modern benchmarks.
\item We introduce ARTA, a joint adversarial framework that internalizes robustness by training detectors against sparsity--constrained worst--case temporal perturbations.
\item Through extensive evaluation on modern benchmarks, we show that ARTA achieves a superior trade--off between detection accuracy, robustness, and interpretability compared to both classical TSAD models and modern deep baselines.
\end{enumerate}

\section{Related Works}

Time--series anomaly detection has recently undergone a major methodological reassessment. For many years, apparent progress in the field was driven by increasingly complex deep learning architectures, ranging from attention--based models to graph neural networks, yet subsequent analyses revealed that much of this progress was illusory \citep{wenig2022current}. A central issue was the widespread use of window--level evaluation schemes, such as Point Adjustment (PA), which substantially inflated performance by tolerating imprecise temporal localization. Under these protocols, models could achieve high scores even when failing to identify anomalies at individual timestamps, obscuring fundamental weaknesses in point--wise detection \citep{kim2022towards}.

This evaluation crisis motivated the development of stricter benchmarks, most notably TSB--AD \citep{liu2024tsbad}. Results on TSB--AD demonstrate that many previously reported SOTA methods degrade sharply when evaluated under refined datasets and strict protocols. Consequently, recent research has shifted away from architectural novelty alone toward methods that emphasize robustness and precise point--wise discrimination under realistic noise conditions.
In response to these stricter standards, two dominant paradigms emerged in the past year. The first is the rise of lightweight pre--trained or foundation models for TSAD. Among these, TSPulse \citep{ekambaram2026tspulse} represents the current state of the art, favoring efficiency over scale. Unlike large language model–based approaches such as MOMENT \citep{goswami2024moment} or Chronos \citep{ansari2024chronos}, TSPulse employs an ultra--compact architecture with approximately one million parameters and learns through masked reconstruction in both time and frequency domains. A multi--head triangulation mechanism combines reconstruction and forecasting deviations to produce anomaly scores, yielding strong performance on the TSB--AD leaderboard and surpassing baselines by a substantial margin. Nevertheless, as a reconstruction--driven method, TSPulse remains vulnerable to the identity--mapping problem: high--capacity decoders can faithfully reconstruct stochastic or high--frequency noise, blurring the distinction between benign perturbations and true point--level anomalies.

A second line of work reframes TSAD through the lens of dynamical systems. FOLD \citep{jhin2026pointwise} departs from conventional reconstruction or forecasting errors and instead models anomaly emergence as the accumulation of system stress. By extracting sensitivity and uncertainty--based stress signals and integrating them through a fold--bifurcation ordinary differential equation, it explicitly targets point--wise anomaly detection as a dynamical tipping phenomenon. This formulation avoids window--level aggregation and achieves strong point--wise performance. However, FOLD relies on calibrated thresholds derived from normal--system dynamics rather than learning an explicit discriminative boundary, leaving it sensitive to noise that mimic early stress signals.

Despite these advancements, a fundamental limitation persists across the current TSAD landscape: robust point--wise detection under noisy conditions remains unresolved. Many contemporary methods implicitly retreat from strict point--level decision--making, instead relying on smoothing, block--level reporting, or post--hoc aggregation to suppress false positives. Reconstruction--based models tend to absorb noise into their learned representations, while dynamical approaches model stability without actively rejecting adversarial perturbations \citep{gu2025imperceptibleadversarialattackstime}. As a result, reliable discrimination between stochastic noise and structurally meaningful anomalies at individual timestamps remains an open challenge, even among leading methods on TSB--AD.

To address this gap, we propose an approach that integrates adversarial training directly into point--wise TSAD. By exposing the model to adversarial perturbations during training and coupling this with a mask--based architecture, we explicitly enforce a robust decision boundary between noise and anomalies at the finest temporal resolution. This design directly targets the noise--mixing failure modes revealed by strict point--wise benchmarks and complements recent advances by adding principled robustness rather than additional architectural complexity.

\section{Methodology}

\begin{figure*}
    \centering
    \Description{A block diagram showing the ARTSADIR architecture.}
    \includegraphics[width=0.85\linewidth]{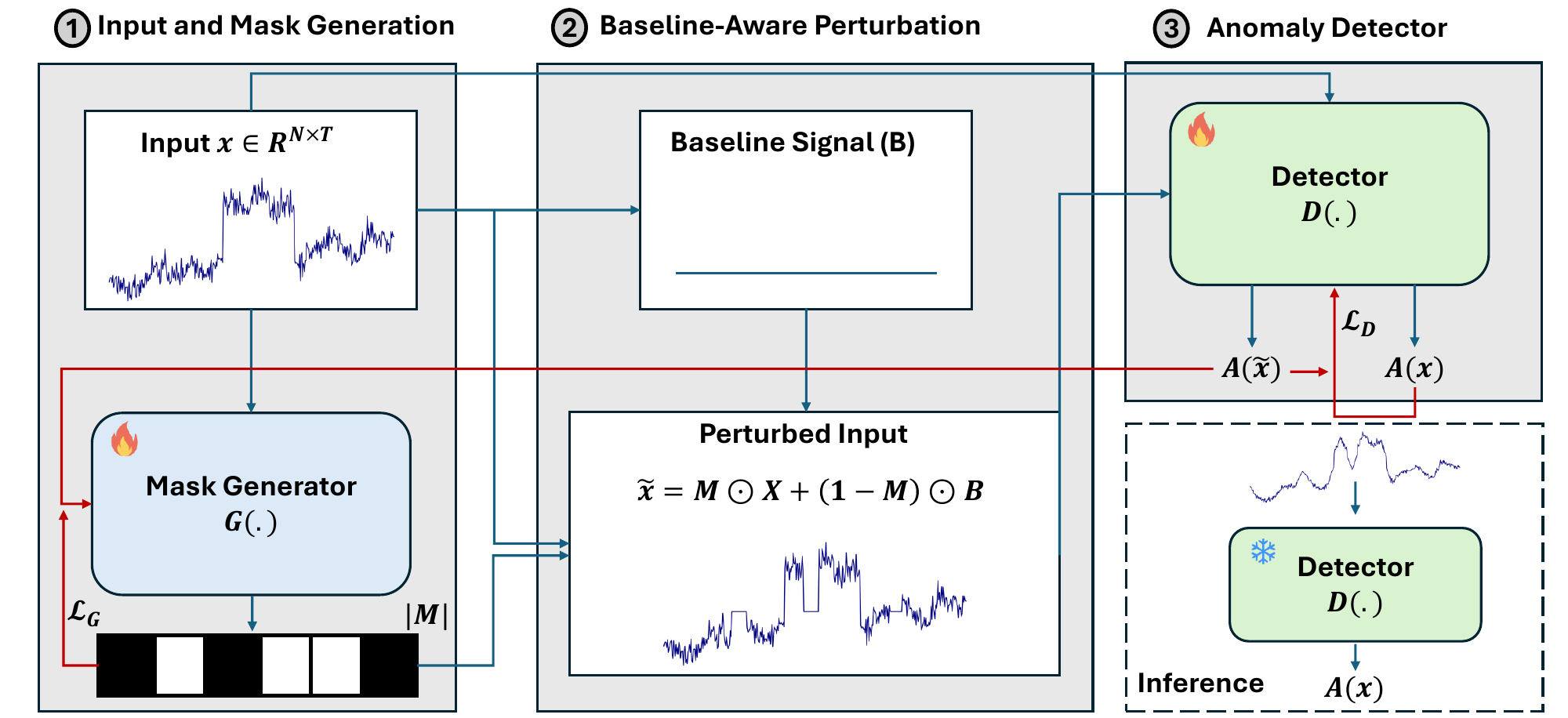}
    \caption{Overview of the proposed method during training and inference.}
    \label{fig:tsadir}
\end{figure*}

\subsection{Problem Setup and Notation}

Let $X \in \mathbb{R}^{N \times T}$ denote a multivariate time series of length $T$ with $N$ sensors. The anomaly detector $D_\theta$ maps an input time--series segment to a scalar anomaly score
\begin{equation}
A_D(X) \in \mathbb{R},
\end{equation}
where larger values indicate a higher likelihood of anomalous behavior.

Importantly, we also consider point--level anomaly scoring. Specifically, the detector produces a sequence of point--wise anomaly scores
\begin{equation}
\mathbf{a}(X) = \{a_t\}_{t=1}^T, \quad a_t \in \mathbb{R},
\end{equation}
where $a_t$ quantifies the anomalousness of the signal at time index $t$. The segment--level anomaly score $A_D(X)$ is obtained by aggregating the point--wise scores over time, i.e.,
\begin{equation}
A_D(X) = \mathcal{G}\big( \{a_t\}_{t=1}^T \big),
\end{equation}
where $\mathcal{G}(\cdot)$ denotes a temporal aggregation operator (e.g., mean or maximum). This formulation enables temporal localization of anomalies at the resolution of individual time points, in contrast to methods that operate exclusively at the segment level and yield a single uninterpretable score per block.
Throughout this work, we focus on semi--supervised anomaly detection, where the detector is trained using only normal data. The anomaly score can be instantiated in different ways depending on the backbone architecture, such as reconstruction error for autoencoder--based models or prediction error for forecasting--based models.

\subsection{Mask Generator and Perturbation Model}

The mask generator $G_\phi$ takes the input time--series and produces a real--valued temporal mask

\begin{equation}
    M = G_\phi(X), \quad M \in [0,1]^T
\end{equation}
\b{where $T$ denotes the size of the mask that} identifies influential time stamps for the detector's anomaly score. We restrict the mask to the temporal dimension to encourage sensor--agnostic temporal robustness and avoid trivial feature--wise suppression.

The mask is broadcast across feature dimensions and applied to the input using a baseline--aware perturbation model:

\begin{equation}
\tilde{X} = X \odot M + (1 - M) \odot B    
\end{equation}
where $B$ is a baseline signal (e.g., per--sensor mean or zero), and $\odot$ denotes element--wise multiplication. This formulation avoids scale artifacts and distribution shifts associated with hard masking. In our implementation, the baseline $B$ is chosen as the per--sensor temporal mean computed within each input sequence, without using future context, resulting in a scale--preserving and in--distribution perturbation.This choice ensures that masked regions are replaced with a statistically neutral reference signal, preventing the introduction of artificial discontinuities that could lead to trivial adversarial solutions. Alternative baseline configurations are explored in Section~\ref{tab:baseline}. Since the generator operates exclusively on normal training data, it learns to identify model--sensitive temporal regions that expose the detector’s vulnerabilities, rather than localizing true anomalous patterns.

While conventional adversarial training typically injects diffuse, unstructured noise, our framework instead constructs temporally structured, sparsity--constrained perturbations that remain domain--aligned and closely reflect real--world failure modes. Diffuse i.i.d. noise is deliberately excluded from the training objective: such noise is already well handled by standard data augmentation, smoothing, and robust optimization techniques, and it is not the dominant source of catastrophic anomaly score instability. Our focus is therefore on the more challenging failure regime of localized, structured perturbations that disrupt short--term dynamics and induce disproportionate changes in anomaly scores.  Additional mathematical proofs are provided in Section~\ref{sec:theory} to offer deeper insight into the theoretical foundations of ARTA's robustness.

\subsection{Generator Objective}

The generator is trained to identify minimal temporal regions that are sufficient to increase the detector's anomaly score. Its objective balances two competing goals:
(i) maximizing the anomaly score of the perturbed input, and
(ii) enforcing sparsity of the mask.
The generator loss is defined as:
\begin{equation}
\mathcal{L}_G(\phi) =
- A_D(\tilde{X}) + \lambda \|M\|_1   
\end{equation}
where $\lambda > 0$ controls the sparsity strength. The first term encourages the generator to find task--relevant perturbations that increase the detector's anomaly score (\textit{e.g.} reconstruction error in autoencoders), while the $\ell_1$ regularization promotes compact and localized masks.

\subsection{Detector Objective}

The detector is trained to minimize the anomaly score for both clean and masked signals, encouraging invariance to generator--induced perturbations. The detector loss is defined as:
\begin{equation}
\mathcal{L}_D(\theta) =
A_D(X) + \gamma A_D(\tilde{X}),
\end{equation}
where $\gamma \geq 0$ controls the emphasis on stability. Note that setting $\gamma = 0$ eliminates the loss component responsible for encouraging robustness of the detector with respect to the generator--induced perturbations.

Minimizing $\mathcal{L}_D$ encourages the detector to rely on distributed temporal features rather than a small set of highly sensitive time steps.

\subsection{Joint Adversarial Optimization}

Figure~\ref{fig:tsadir} illustrates the joint training and inference procedure. The overall training objective is formulated as a min--max optimization problem:
\begin{equation}
\min_{\theta} \max_{\phi} \;
\mathbb{E}_{X \sim \mathcal{D}_{\text{normal}}}
\left[
\mathcal{L}_D(\theta) - \mathcal{L}_G(\phi)
\right].
\end{equation}

In practice, we adopt an alternating optimization strategy in which the generator and detector are updated in turn using stochastic gradient descent. This adversarial interaction encourages the generator to identify increasingly informative sparse temporal masks, while simultaneously driving the detector to become more stable under such structured perturbations. To promote training stability, we first pretrain the detector for a small number of warm--up epochs before initiating joint adversarial training.

At inference time, anomaly scores are computed exclusively using the detector, and the generator does not participate in the decision rule. Instead, the generator acts as a training--time adversary that exposes model sensitivity and enforces robustness during learning. We explore mask--aware scoring variants that incorporate the learned masks as relevance weights or quantify detector sensitivity gaps (Section~\ref{sec:scores}); however, detector--only scores are used as the default throughout our experiments.

\section{Theoretical Analysis}

\subsection{Theoretical Foundations of Stability}
\label{sec:theory}

We analyze how sparsity--constrained masking induces stability of anomaly scores under localized input perturbations. Our analysis is deterministic and does not rely on probabilistic assumptions or robustness certification guarantees.

\paragraph{Setup.}
Let $A_D : \mathbb{R}^{N \times T} \rightarrow \mathbb{R}$ denote the anomaly scoring function implemented by the detector $D_\theta$. We assume that $A_D$ is Lipschitz continuous with respect to the $\ell_1$ norm; that is, there exists a constant $L > 0$ such that
\begin{equation}
\label{eq:lipschitz}
|A_D(X_1) - A_D(X_2)| \le L \|X_1 - X_2\|_1
\end{equation}
for all $X_1, X_2$ in the input domain. While this assumption is standard in stability analyses, we explicitly satisfy it in practice by applying Spectral Normalization to the weight matrices of the detector \cite{miyato2018spectral}.

\begin{theorem}[Stability Under Sparse Baseline--Aware Masking]
\label{thm:stability}
Let $\delta \in \mathbb{R}^{N \times T}$ satisfy $\|\delta\|_\infty \le \varepsilon$, and let $M \in [0,1]^T$ be a fixed temporal mask. Define the perturbed masked input as
\begin{equation}
\tilde X_\delta = (X + \delta) \odot M + (1 - M) \odot B.
\end{equation}
Then the anomaly score satisfies the bound
\begin{equation}
\left| A_D(\tilde X_\delta) - A_D(\tilde X) \right|
\le L \, \varepsilon \, \|M\|_1 .
\end{equation}
\end{theorem}

\begin{proof}
By Lipschitz continuity of $A_D$ in~\eqref{eq:lipschitz},
\[
\left| A_D(\tilde X_\delta) - A_D(\tilde X) \right|
\le L \| \tilde X_\delta - \tilde X \|_1.
\]
Substituting the definition of $\tilde X_\delta$ yields
\[
\tilde X_\delta - \tilde X = \delta \odot M,
\]
since the baseline term cancels. Because $\|\delta\|_\infty \le \varepsilon$ and the mask is broadcast across feature dimensions, we obtain
\[
\| \delta \odot M \|_1
\le \varepsilon \sum_{t=1}^{T} M_t
= \varepsilon \|M\|_1,
\]
where constant factors related to the number of sensors are absorbed into the Lipschitz constant $L$. Combining these results proves the claim.
\end{proof}

\paragraph{Interpretation.}
The bound shows that the detector’s sensitivity to bounded input perturbations is controlled by the sparsity of the temporal mask. In particular, sparser masks induce tighter upper bounds on the variability of the anomaly score under localized corruptions. While this result does not provide global robustness guarantees, it formalizes how explanation compactness constrains the effective perturbation subspace and upper--bounds worst--case score deviations.

Importantly, this stability result is conditional on the detector exhibiting bounded sensitivity. During training, sparsity--constrained adversarial masking restricts the space of perturbations explored by the generator, acting as an implicit structural regularizer that encourages the detector to rely on distributed temporal evidence rather than highly localized, brittle features.

\subsection{Masked Perturbations as a Structured Function Class}
\label{sec:structure}

This section analyzes the space of perturbed inputs induced by sparsity--constrained masking. Rather than adopting a probabilistic interpretation, we treat masking as a deterministic structural operation that restricts how the input time series can be modified.

Let $X \in \mathbb{R}^{N \times T}$ denote a bounded multivariate time series and let $B \in \mathbb{R}^{N \times T}$ be a fixed baseline signal. For a mask $M \in [0,1]^T$, we define the masked input as
\[
\tilde{X}(M) = X \odot M + (1 - M) \odot B.
\]

We denote by $\mathcal{M}_k = \{ M \in [0,1]^T \mid \|M\|_1 \leq k \}$ the set of masks with bounded $\ell_1$ norm.

\begin{theorem}[Bounded Perturbation Capacity Under Sparse Masking]
\label{prop:bounded_capacity}
Let $X$ and $B$ be bounded signals. Then the family of masked inputs
\[
\mathcal{X}_k = \{ \tilde{X}(M) \mid M \in \mathcal{M}_k \}
\]
has bounded diameter under the $\ell_1$ norm:
\[
\sup_{\tilde{X}_1, \tilde{X}_2 \in \mathcal{X}_k}
\| \tilde{X}_1 - \tilde{X}_2 \|_1
\leq 2k \| X - B \|_\infty.
\]
\end{theorem}

\begin{proof}
Let $M_1, M_2 \in \mathcal{M}_k$. Then
\[
\| \tilde{X}(M_1) - \tilde{X}(M_2) \|_1
= \| (X - B) \odot (M_1 - M_2) \|_1.
\]
Applying Hölder's inequality yields
\[
\| (X - B) \odot (M_1 - M_2) \|_1
\leq \| X - B \|_\infty \| M_1 - M_2 \|_1.
\]
Since $\|M_1\|_1, \|M_2\|_1 \leq k$, we have
\[
\| M_1 - M_2 \|_1 \leq \|M_1\|_1 + \|M_2\|_1 \leq 2k,
\]
Therefore,
\[
\| \tilde{X}_1 - \tilde{X}_2 \|_1 \leq 2k \|X--b\|_{\infty} \qquad \forall \tilde{X}_1,\tilde{X}_2 \in \mathcal{X}_k,
\]
Which can be written as
\[
\sup_{\tilde{X}_1, \tilde{X}_2 \in \mathcal{X}_k}
\| \tilde{X}_1 - \tilde{X}_2 \|_1
\leq 2k \| X - B \|_\infty.
\]
which concludes the proof.
\end{proof}

\begin{theorem}[Bounded Anomaly Score Variability]
\label{boundedAV}
Let $A_D$ be $L$-Lipschitz with respect to the $\ell_1$ norm. Then for any $k > 0$, the anomaly score variation over the family of sparsity--constrained masked inputs
\[
\mathcal{X}_k = \{ X \odot M + (1 - M) \odot B \mid \|M\|_1 \le k \}
\]
is bounded as
\begin{equation}
\sup_{\tilde X_1, \tilde X_2 \in \mathcal{X}_k}
\left| A_D(\tilde X_1) - A_D(\tilde X_2) \right|
\le 2 L k \|X - B\|_\infty .
\end{equation}
\end{theorem}

\begin{proof}
By Lipschitz continuity of $A_D$ and Theorem~2,
\[
\left| A_D(\tilde X_1) - A_D(\tilde X_2) \right|
\le L \|\tilde X_1 - \tilde X_2\|_1
\le 2 L k \|X - B\|_\infty.
\]
\end{proof}

\paragraph{Implications for Stability of Anomaly Scores.}

Theorem~\ref{prop:bounded_capacity} shows that enforcing sparsity on the mask restricts the variability of masked inputs. Combined with the Lipschitz continuity of the anomaly detector (Theorem~\ref{boundedAV}), this implies that the anomaly score cannot vary arbitrarily within the family $\mathcal{X}_k$.

In particular, sparse masking constrains the generator to explore a limited perturbation space, preventing highly distributed or large--magnitude corruptions. This structural restriction acts as an implicit regularizer, encouraging stability of the detector with respect to localized temporal perturbations.

\section{Experiments}

In this section, we evaluate ARTA through a series of experiments on the rigorous TSB--AD benchmark~\citep{liu2024tsbad}. TSB--AD comprises multiple datasets carefully curated to mitigate the evaluation pitfalls and flawed labels common in legacy TSAD benchmarks, ensuring a standardized and high--fidelity assessment. In all experiments, we utilize an LSTM--based autoencoder as the detector, chosen for its inherent inductive bias toward temporal continuity, and an LSTM as the generator. An empirical validation of this backbone choice against alternative architectures is provided in Section~\ref{Backbone}.

\subsection{Comparison with State--of--the--Art}

We first compare ARTA against several state--of--the--art anomaly detection methods. To ensure the integrity of our benchmarking, we utilized the implementations provided by the TSB--AD benchmark where available. For all other methods, we relied on the original implementations and hyperparameters specified in their respective papers. For evaluation, we use VUS--PR, which is well--suited for time--series anomalies that span multiple consecutive time steps. Standard point--based metrics, such as Precision, Recall, F--score, or AUC--ROC/AUC--PR, often fail to capture the temporal extent of anomalies and are sensitive to noise or small misalignments between predicted and true ranges. VUS--PR~\citep{paparrizos2022volume} addresses these limitations by providing a threshold--independent, parameter--free assessment. It constructs a 3D surface over all thresholds and buffer sizes and computes its volume, offering a holistic measure of detection performance for both point- and range--based anomalies. We report performance using other evaluation metrics in Appendix~\ref{app:metrics}.
Table~\ref{tab:sota} summarizes the performance of ARTA in comparison with state--of--the--art methods. Our approach consistently achieves high VUS--PR scores across the benchmark, demonstrating its ability to accurately detect temporally extended anomalies across diverse datasets, which vary in anomaly types, durations, and occurrence ratios. We observe that ARTA tops the performance chart in 7 out of 10 datasets and comes second in the remaining three. Even for the datasets where it ranks second, our subsequent experiments in Section~\ref{robustness} show potential improvements in robustness.
These results can be attributed to the training procedure of ARTA. The adversarial training encourages the detector to remain stable against mask--guided perturbations, driving the model to focus on distributed temporal patterns rather than isolated point values. This enhances robustness to non--anomalous fluctuations and noise, allowing the model to reliably distinguish true anomalies during inference. Consequently, ARTA captures complex, distributed temporal features that are critical for accurate and robust anomaly detection across diverse datasets. The following sections further analyze complementary ablations and anomaly scoring strategies to isolate the contributions of individual components.

\begin{table*}
\small
    \centering
    \begin{tabular}{ccccccccccc}
\hline
Method & CATSv2 & Daphnet & Exathlon & GECCO & LTDB & MITDB & OPP & PSM & SMD & SVDB \\
\hline
PCA~\cite{jolliffe1986principal} & 0.12 & 0.13 & \textbf{0.95} & \textbf{0.20} & 0.24 & 0.07 & 0.30 & 0.16 & 0.36 & 0.11 \\
MCD~\cite{mcd} & 0.13 & 0.14 & 0.80 & 0.03 & 0.21 & 0.04 & 0.17 & 0.26 & 0.26 & 0.07 \\
OCSVM~\cite{ocsvm4ts}  & 0.08 & 0.06 & 0.83& 0.04 & 0.20 & 0.04 & 0.12 & 0.19 & 0.28 & 0.06 \\
KNN~\cite{knn} & 0.07 & 0.25 & 0.33 & 0.11 & 0.19 & 0.04 & 0.06 & 0.12 & 0.30 & 0.06 \\
LOF~\cite{breunig2000lof} & 0.05 & 0.11 & 0.16 & 0.13 & 0.19 & 0.04 & 0.10 & 0.15 & 0.16 & 0.06 \\
KMeansAD~\cite{kmeansad} & 0.12 & 0.30 & 0.37 & 0.06 & 0.41 & 0.06 & 0.06 & 0.21 & 0.36 & 0.20 \\
CBLOF~\cite{cblof} & 0.06 & 0.10 & 0.86 & 0.03 & 0.20 & 0.04 & 0.14 & 0.19 & 0.22 & 0.07 \\
IForest~\cite{iforest} & 0.05 & 0.13 & 0.35 & 0.04 & 0.21 & 0.04 & 0.18 & 0.19 & 0.26 & 0.07 \\
HBOS~\cite{hbos} & 0.05 & 0.15 & 0.32 & 0.04 & 0.21 & 0.04 & 0.17 & 0.17 & 0.25 & 0.07 \\
AutoEncoder~\cite{autoencoder} & 0.06 & 0.13 & 0.91 & 0.05 & 0.21 & 0.04 & 0.14 & \underline{0.28} & 0.30 & 0.06 \\
LSTMAD~\cite{lstmad} & 0.04 & 0.31 & 0.82 & 0.02 & 0.30 & 0.09 & 0.17 & 0.24 & 0.33 & 0.15 \\
DeepAnT~\cite{munir2018deepant} & 0.08 & 0.21 & 0.68 & 0.03 & 0.33 & \underline{0.14} & 0.16 & 0.22 & 0.35 & 0.19 \\
Donut~\cite{donut} & 0.07 & 0.17 & 0.66 & 0.03 & 0.26 & 0.12 & 0.15 & 0.20 & 0.19 & 0.11 \\
RobustPCA~\cite{robustpca} & 0.04 & 0.06 & 0.77 & 0.02 & 0.23 & 0.04 & 0.13 & 0.12 & 0.10 & 0.08 \\
OmniAnomaly~\cite{OmniAnomaly} & 0.04 & 0.34 & 0.84 & 0.02 & 0.44 & 0.11 & 0.18 & 0.16 & 0.17 & 0.35 \\
EIF~\cite{eif} & 0.06 & 0.15 & 0.41 & 0.04 & 0.19 & 0.04 & 0.10 & 0.18 & 0.32 & 0.07 \\
COPOD~\cite{li2020copod} & 0.05 & 0.11 & 0.40 & 0.04 & 0.21 & 0.04 & 0.17 & 0.20 & 0.19 & 0.07 \\
USAD~\cite{audibert2020usad} & 0.04 & 0.34 & 0.84 & 0.02 & 0.41 & 0.12 & 0.18 & 0.19 & 0.16 & 0.32 \\
AnomalyTransformer~\cite{xu2021anomaly} & 0.03 & 0.07 & 0.10 & 0.02 & 0.21 & 0.05 & 0.07 & 0.21 & 0.07 & 0.08 \\
TimesNet~\cite{wu2022timesnet} & 0.07 & 0.27 & 0.42 & 0.03 & 0.27 & 0.07 & 0.06 & 0.14 & 0.14 & 0.11 \\
TranAD~\cite{tuli2022tranad} & 0.04 & 0.31 & 0.10 & 0.02 & 0.26 & 0.07 & 0.16 & 0.23 & 0.30 & 0.12 \\
FITS~\cite{xu2023fits} & 0.13 & 0.33 & 0.63 & 0.03 & 0.23 & 0.05 & 0.05 & 0.13 & 0.17 & 0.10 \\
OFA~\cite{zhou2023one} & 0.13 & 0.31 & 0.58 & 0.04 & 0.29 & 0.06 & 0.05 & 0.17 & 0.17 & 0.12 \\
TSPulse (ZS)~\cite{ekambaram2026tspulse} & 0.05 & 0.35 & 0.89 & 0.17 & 0.36 & 0.07 & 0.07 & 0.14 & 0.35 & 0.38 \\
TSPulse (FT)~\cite{ekambaram2026tspulse} & 0.07 & 0.35 & 0.91 & \underline{0.18} & \textbf{0.57} & \underline{0.14} & 0.07 & 0.14 & 0.36 & \underline{0.47} \\
FOLD~\cite{jhin2026pointwise} & \underline{0.23} & \underline{0.39} & \underline{0.93} & 0.08 & 0.37 & 0.09 & \textbf{0.82} & 0.19 & \textbf{0.46} & 0.39 \\
\textbf{ARTA} & \textbf{0.31} & \textbf{0.41} & \textbf{0.95} & \textbf{0.20} & \underline{0.45} & \textbf{0.21} & \underline{0.35} & \textbf{0.30} & \underline{0.37} & \textbf{0.48} \\
\hline

    \end{tabular}
    \caption{Comparison of VUS--PR ($\uparrow$) between ARTA and other baselines on 10 Multivariate Time--Series Anomaly Detection Datasets.}
    \label{tab:sota}
\end{table*}

\subsection{Robustness Evaluation}
\label{robustness}
To evaluate the robustness of ARTA under input perturbations, we compare its performance against three baselines from previous experiments on three TSB--AD datasets~\citep{liu2024tsbad}. We assess robustness by injecting additive noise into the test time series and measuring how detection performance degrades as the perturbation strength increases. We assess the robustness of ARTA under two challenging noise regimes that mimic real--world corruptions. In addition, we evaluate robustness to less challenging Gaussian noise in Appendix~\ref{app:nongauss}.In all our experiments in this section, noise is applied only at test time and not during training, ensuring that the evaluation isolates the robustness of the trained model rather than confounding it with noise--aware learning.

\textbf{Salt--and--Pepper Noise:} Salt--and--pepper noise randomly replaces a fraction $p$ of time--series values with extreme minimum or maximum values. For evaluation, we vary $p$ from 1\% (mild) to 20\% (severe), and plot the VUS--PR with respect to the corruption probability in Figure~\ref{fig:robust_noise}.We observe that ARTA exhibits significantly slower performance degradation compared to the other baselines. Salt--and--pepper noise, which simulates sparse impulsive corruption, particularly challenges methods that rely on a few isolated points for anomaly detection. By being trained adversarially to withstand such corruptions, ARTA maintains superior robustness, especially in the high--corruption regime.

\textbf{Colored Noise:} Colored noise exhibits temporal correlations and is more challenging than uncorrelated Gaussian noise. We generate additive colored noise using a first--order autoregressive process:
\[
n_t = \rho n_{t--1} + \epsilon_t, \quad \epsilon_t \sim \mathcal{N}(0, \sigma^2), \quad |\rho|<1,
\]
where $\rho$ controls the temporal correlation strength and $\epsilon_t$ is white Gaussian noise. The noisy time series is then:
\[
\tilde{X} = X + n.
\]
The SNR for colored noise is defined as:
\[
\text{SNR}_{\text{col}} = 10 \log_{10} \frac{\mathbb{E}[\|X\|^2]}{\mathbb{E}[\|n\|^2]},
\]
where the noise power $\mathbb{E}[\|n\|^2]$ accounts for the temporal correlation. We vary $\sigma$ to achieve target SNR values from 30 dB to 10 dB, and set $\rho = 0.5$ to reflect moderate temporal correlation. The resulting graph is shown in Figure~\ref{fig:robust_noise}.
We observe that colored noise causes far greater performance degradation than salt--and--pepper noise. Most methods deteriorate rapidly as the noise strength increases. A key characteristic of colored noise is that it introduces temporal drift, which makes many methods particularly sensitive to this form of corruption. This behavior is especially concerning in real--world monitoring systems, where drift commonly arises from sensor aging, calibration errors, or gradual environmental changes. In contrast, our method exhibits a more graceful degradation and consistently outperforms the baselines, even in the low--SNR regime, highlighting its suitability for long--term deployment in realistic, noisy environments.
The robustness of ARTA to different types of noise can be attributed to its adversarial training mechanism. During training, the detector is exposed to worst--case, mask--guided perturbations that are explicitly optimized to maximally affect the anomaly score. These perturbations are structured, coherent, and aligned with the detector’s most sensitive directions, making them more challenging than stochastic noise. As a result, the detector learns to rely on stable, distributed temporal patterns rather than brittle point--wise deviations.

\begin{figure*}[t]
    \centering
    \Description{A figure showing the results of robustness evaluation of anomaly detection methods under input corruption.}
    \begin{subfigure}{0.32\linewidth}
        \centering
        \includegraphics[width=0.8\linewidth]{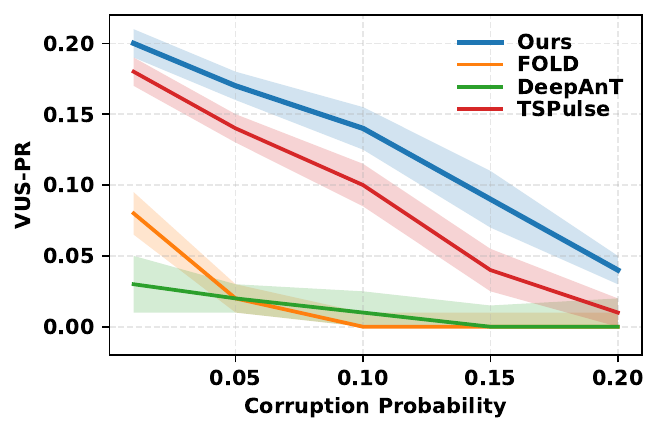}
        \caption{GECCO}
        \label{fig:robust_gecco_sp}
    \end{subfigure}
    \hfill
    \begin{subfigure}{0.32\linewidth}
        \centering
        \includegraphics[width=0.8\linewidth]{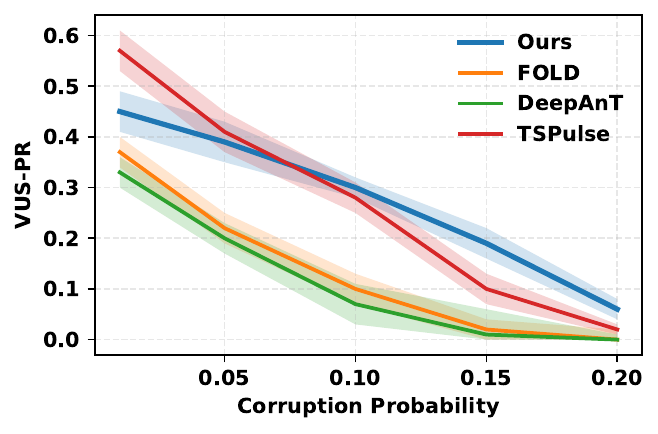}
        \caption{LTDB}
        \label{fig:robust_ltdb_sp}
    \end{subfigure}
    \hfill
    \begin{subfigure}{0.32\linewidth}
        \centering
        \includegraphics[width=0.8\linewidth]{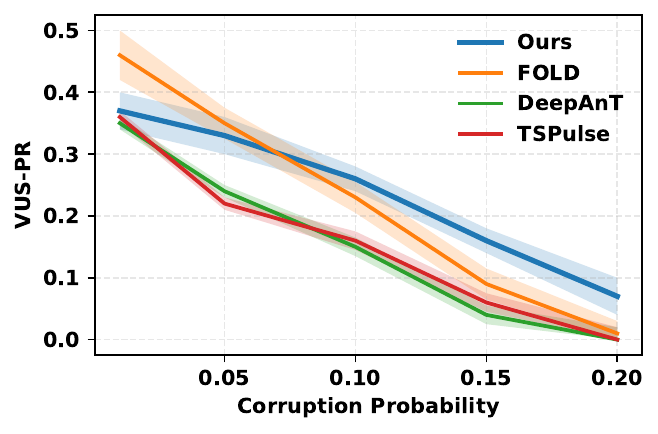}
        \caption{SMD}
        \label{fig:robust_smd_sp}
    \end{subfigure}

    \vspace{0.8em}

    \begin{subfigure}{0.32\linewidth}
        \centering
        \includegraphics[width=0.8\linewidth]{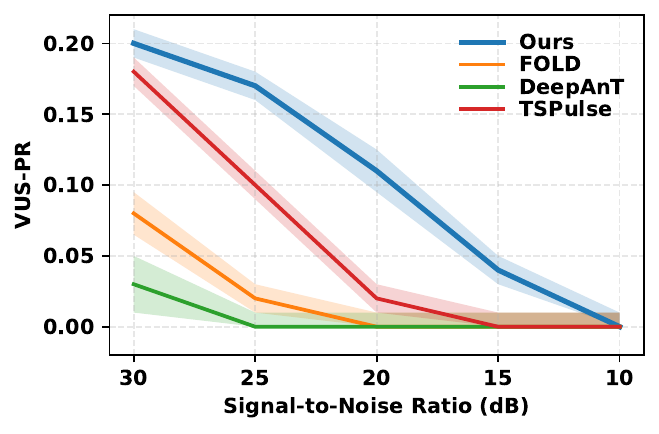}
        \caption{GECCO}
        \label{fig:robust_gecco_colored}
    \end{subfigure}
    \hfill
    \begin{subfigure}{0.32\linewidth}
        \centering
        \includegraphics[width=0.8\linewidth]{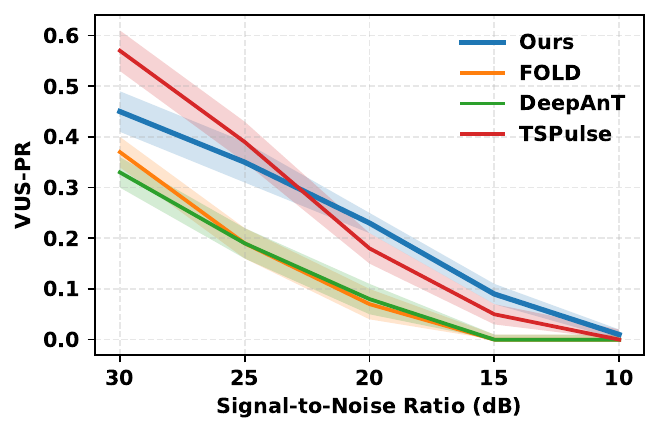}
        \caption{LTDB}
        \label{fig:robust_ltdb_colored}
    \end{subfigure}
    \hfill
    \begin{subfigure}{0.32\linewidth}
        \centering
        \includegraphics[width=0.8\linewidth]{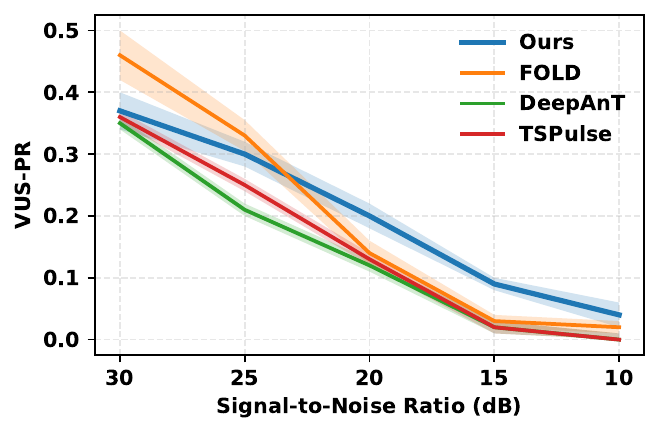}
        \caption{SMD}
        \label{fig:robust_smd_colored}
    \end{subfigure}

    \caption{
    Robustness evaluation of anomaly detection methods under input corruption.
    \textbf{Top row:} Salt--and--pepper noise with varying corruption probability.
    \textbf{Bottom row:} Colored noise with varying SNR.
    Results are shown for three datasets and averaged over five runs.
    }
    \label{fig:robust_noise}
\end{figure*}

\subsection{Ablation Study}

To evaluate the contribution of individual components in ARTA, we consider the following variants: \textbf{No Generator:} In this variant, the generator is completely removed, and only the detector (LSTM--based AE) is trained and evaluated. This ablation assesses the contribution of the generator and adversarial training by measuring performance when only the detector is trained.
\textbf{No Adversarial Training:}
This variant generates masks that perturb the input. However, the detector is not explicitly trained to withstand worst--case, model--aware perturbations. In this setting, the generator is static and does not adaptively produce masks to challenge the detector. 
\textbf{No Sparsity Penalty:}  
Here, the full adversarial training framework is retained, but the $\ell_1$ sparsity penalty in the generator's loss is removed. This evaluates the effect of enforcing compact, localized masks on detection performance.

These ablations allow us to systematically study the role of adversarial training, sparsity regularization, and detector robustness. Table~\ref{tab:ablation_comp} summarizes the results of each variant. As seen in this table, removing the generator results in substantial performance degradation across most datasets. In this setting, the model reduces to a standard LSTM--based autoencoder that relies solely on reconstruction error for anomaly scoring. While reconstruction--based methods can capture coarse deviations, they are known to be sensitive to noise and often fail to distinguish subtle or temporally distributed anomalies. The observed drop in performance confirms that adversarially guided perturbations are essential to learn discriminative temporal representations relevant to tasks.
Although the \emph{w/o Adversarial Training} variant consistently outperforms the \emph{w/o Generator} baseline, it remains inferior to the full model, highlighting the importance of adversarial interaction. Without the adversarial objective, the generator lacks incentive to identify truly sensitive temporal regions, and the detector does not learn robustness to task--relevant corruptions.
Finally, the \emph{w/o Sparsity Penalty} variant isolates the role of enforcing compact and localized masks. When the $\ell_1$ regularization on the generator output is removed, the generator tends to produce diffuse or near--uniform masks, perturbing large portions of the input sequence. While such perturbations can increase the detector’s anomaly score during training, they are less informative and weaken the adversarial signal. As a result, the detector learns robustness primarily to coarse, global distortions rather than precise temporal patterns. The performance drop observed in this variant demonstrates that sparsity is a crucial ingredient for effective adversarial training and accurate anomaly detection.
Taken together, these ablations validate the core design principles of ARTA. The synergy between these components enables ARTA to capture distributed temporal dependencies that are critical for detecting range--based anomalies, explaining its superior performance across datasets.

\begin{table*}
    \centering
    \begin{tabular}{ccccccccccc}
\hline
Method & CATSv2 & Daphnet & Exathlon & GECCO & LTDB & MITDB & OPP & PSM & SMD & SVDB \\
\hline
Full Model & \textbf{0.31} & \textbf{0.41} & \textbf{0.95} & \textbf{0.20} & \textbf{0.45} & \textbf{0.21} & \textbf{0.35} & \textbf{0.30} & \textbf{0.37} & \textbf{0.48} \\
w/o Generator & 0.08 & 0.13 & 0.90 & 0.06 & 0.23 & 0.05 & 0.17 & 0.28 & 0.26 & 0.11\\
w/o Adversarial Training & 0.13 & 0.34 & 0.93 & 0.13 & 0.14 & 0.11 & 0.26 & 0.28 & 0.31 & 0.27\\
w/o Sparsity Penalty & 0.18 & 0.28 & 0.90 & 0.18 & 0.26 & 0.18 & 0.22 & 0.18 & 0.28 & 0.36 \\
\hline
Reconstruction Error & \textbf{0.31} & \textbf{0.41} & \textbf{0.95} & \textbf{0.20} & \textbf{0.45} & \textbf{0.21} & 0.35 & \textbf{0.30} & \textbf{0.37} & \textbf{0.48} \\
Mask--Weighted & \textbf{0.31} & 0.36 & 0.88 & 0.08 & 0.31 & 0.12 & \textbf{0.69} & 0.29 & 0.30 & 0.29 \\
Sensitivity--Gap & 0.25 & 0.11 & 0.25 & 0.09 & 0.23 & 0.12 & 0.07 & 0.15 & 0.09 & 0.22\\
\hline

    \end{tabular}
    \caption{Ablation study results. The top and bottom rows show the impact of components and anomaly scores, respectively.}
    \label{tab:ablation_comp}
\end{table*}

\subsubsection{Anomaly Scoring Strategies}
\label{sec:scores}

We investigate three strategies for computing anomaly scores at test time. Let $X$ denote an input window, $\hat X$ its reconstruction from the detector, and $M$ the generator mask. We can define our anomaly scores as:
\textbf{Detector--only:}  
The anomaly score is computed from the reconstruction error alone:
\[
S_{\text{detector}}(X) = \frac{1}{T N} \sum_{t=1}^{T} \sum_{n=1}^{N} (X_{t,n} - \hat X_{t,n})^2.
\]
It measures the deviation from the learned normal behavior, ignoring the generator mask.
\textbf{Mask--weighted:}  
The reconstruction error is weighted by the generator mask:
\[
S_{\text{MW}}(X) = \frac{1}{T N} \sum_{t=1}^{T} \sum_{n=1}^{N} M_t (X_{t,n} - \hat X_{t,n})^2.
\]
This emphasizes temporal regions that the generator identifies as most relevant to the detector.
\textbf{Sensitivity--gap:}  
We compute the difference in reconstruction error between the original and masked input:
\[
S_{\text{SG}}(X) = \frac{1}{T N} \sum_{t=1}^{T} \sum_{n=1}^{N} \big( (X_{t,n} - \hat X_{t,n})^2 - (X_{t,n} - \hat X_{t,n}^{\tilde X})^2 \big),
\]
where $\tilde X = M \odot X + (1--M) \odot B$ is the masked input. This strategy captures both the abnormality and the sensitivity of the detector to perturbations induced by the generator.

The comparison of scoring strategies in Table~\ref{tab:ablation_comp} highlights the interaction between the detector’s reconstructive robustness and the generator’s sparsity--constrained perturbations. Across most benchmarks, the \emph{Reconstruction Error} yields the strongest performance, indicating that the detector effectively learns stable representations of both local and long--range temporal structure during adversarial training. In this regime, the mask generator primarily serves as a training--time regularizer: by enforcing invariance to structured perturbations, it enables reconstruction residuals to become a reliable indicator of normality. Since the generator is trained only on normal data and primarily targets vulnerable regions of the learned manifold, its contribution at inference is limited. As a result, it can be safely removed during deployment, preserving performance while eliminating inference--time overhead.
An exception arises on the OPP dataset, where the \emph{mask--weighted} score substantially outperforms pure reconstruction. OPP contains complex human activities in which anomalies closely resemble normal behavior but differ in subtle, high--frequency temporal patterns. In such cases, reconstruction errors remain small because anomalous signals lie near the learned manifold. The mask--weighted score functions as a learned attention mechanism, amplifying decision--critical deviations that reconstruction alone may suppress. This suggests that for datasets with manifold--aligned or highly localized anomalies, incorporating generator--informed scoring provides essential contrastive information beyond what reconstruction can capture.

\subsection{Effect of Backbone Architecture}
\label{Backbone}

An interesting question to explore is the extent to which the observed performance gains depend on specific architectural choices. In the main paper, both the generator and detector are instantiated as LSTM--based models, reflecting a common design choice for temporal modeling. In this section, we investigate the sensitivity of our framework to the backbone architecture of both components.

Specifically, we study the impact of replacing recurrent architectures with feedforward, convolutional, and transformer--based alternatives, and the effect of using reconstruction--based versus prediction--based detectors within the same adversarial training framework. These experiments aim to assess whether the benefits of adversarial masking and joint training arise from architectural inductive biases or from the proposed learning principle itself.

Table~\ref{tab:backbone_arch} reports the performance across different generator--detector backbone combinations. Several observations emerge.

First, although LSTM--based models consistently achieve the strongest overall performance, the proposed framework remains effective across different architectural choices. Notably, CNN--based variants maintain competitive results, indicating that explicit recurrence is not strictly necessary for adversarial masking to be effective. In contrast, fully connected architectures suffer a noticeable performance degradation, suggesting that a temporal inductive bias is essential for stable adversarial interaction. Interestingly, transformer--based architectures do not provide any advantage over LSTM or CNN backbones. While attention--based models excel at capturing global context, their permutation--invariant formulation can overlook the local temporal continuity required for precise anomaly localization. Consequently, in our experiments, we deliberately adopt a stability--regularized LSTM backbone to preserve the causal structure of multivariate time series while enforcing the Lipschitz constraints required by our stability analysis.

Second, Table~\ref{tab:recon_pred} compares reconstruction--based and prediction--based detectors. Reconstruction--based detectors generally outperform prediction--based ones, especially under severe distribution shifts. This can be attributed to the fact that prediction--based models may implicitly smooth anomalies when forecasting over short horizons, whereas reconstruction error more directly captures deviations induced by adversarial masking.

Overall, these results suggest that while architectural choices influence absolute performance, the gains from adversarial training and learned masking are robust and not tied to a specific backbone or detector formulation.

\begin{table*}
    \centering
    \begin{tabular}{ccccccccccc}
\hline
Method & CATSv2 & Daphnet & Exathlon & GECCO & LTDB & MITDB & OPP & PSM & SMD & SVDB \\
\hline
LSTM & \textbf{0.31} & \textbf{0.41} & \textbf{0.95} & 0.20 & \textbf{0.45} & \textbf{0.21} & 0.35 & \textbf{0.30} & 0.37 & \textbf{0.48} \\
CNN & 0.28 & 0.38 & 0.93 & \textbf{0.22} & 0.41 & 0.18 & \textbf{0.41} & 0.28 & \textbf{0.41} & 0.42 \\
Transformer & 0.24 & 0.38 & \textbf{0.95} & 0.16 & \textbf{0.45} & 0.12 & 0.38 & 0.22 & 0.28& 0.38\\
FC & 0.23 & 0.27 & 0.92 & 0.20 & 0.38 & 0.14 & 0.32 & 0.20 & 0.28 & 0.33 \\
\hline
    \end{tabular}
\caption{Effect of generator and detector backbone architectures. All results are reported in terms of VUS--PR.}
\label{tab:backbone_arch}
\end{table*}

\begin{table*}
    \centering
    \begin{tabular}{ccccccccccc}
\hline
Method & CATSv2 & Daphnet & Exathlon & GECCO & LTDB & MITDB & OPP & PSM & SMD & SVDB \\
\hline
Reconstruction & \textbf{0.31} & \textbf{0.41} & \textbf{0.95} & \textbf{0.20} & \textbf{0.45} & \textbf{0.21} & 0.35 & \textbf{0.30} & \textbf{0.37} & \textbf{0.48} \\
Prediction & 0.28 & 0.34 & \textbf{0.95} & \textbf{0.20} & 0.38 & 0.18 & \textbf{0.37} & 0.22 & 0.33 & 0.40 \\
\hline
    \end{tabular}
\caption{Comparison between reconstruction--based and prediction--based detectors under the same adversarial training setup. Both methods use LSTM as their backbone architecture.}
\label{tab:recon_pred}
\end{table*}

\subsection{Effect of Baseline--Aware Perturbation}

We study the impact of incorporating a baseline signal in the mask--based perturbation model used by ARTA. Recall that our framework applies the generator--produced mask using a baseline--aware formulation,
\[
\tilde{X} = X \odot M + (1 - M) \odot B,
\]
where $B$ represents a reference signal (e.g., per--sensor mean). To isolate the role of the baseline, we compare this formulation against a variant where the baseline is removed, i.e., masked regions are directly suppressed via $\tilde{X} = X \odot M$.

Table~\ref{tab:baseline} reports the anomaly detection performance across multiple datasets. Incorporating a baseline consistently improves performance across all benchmarks. In contrast, removing the baseline leads to noticeable degradation, particularly on datasets with complex temporal structure.

This behavior can be attributed to the fact that hard masking introduces unnatural distribution shifts and scale artifacts, which the detector may exploit as spurious cues. Baseline--aware perturbations preserve the overall signal structure and keep masked inputs on--manifold, enabling the adversarial training objective to more effectively encourage robustness and reliance on distributed temporal patterns rather than isolated point--wise deviations.

\begin{table*}
    \centering
    \begin{tabular}{ccccccccccc}
\hline
Anomaly Score & CATSv2 & Daphnet & Exathlon & GECCO & LTDB & MITDB & OPP & PSM & SMD & SVDB \\
\hline
with Baseline & \textbf{0.31} & \textbf{0.41} & \textbf{0.95} & \textbf{0.20} & \textbf{0.45} & \textbf{0.21} & 0.35 & \textbf{0.30} & \textbf{0.37} & \textbf{0.48} \\
w/o Baseline & 0.29 & 0.35 & 0.93 & 0.14 & 0.42 & 0.20 & \textbf{0.37} & 0.28 & 0.36 & 0.45 \\

\hline

    \end{tabular}
    \caption{Effect of baseline--aware perturbation}
    \label{tab:baseline}
\end{table*}

\subsection{Interpretation of Adversarial Masks}

\begin{figure}[t!]
    \centering
    \Description{A figure showing qualitative examples of generator masks on selected samples.}
    \begin{subfigure}[b]{0.48\linewidth}
        \centering
        \includegraphics[width=\linewidth]{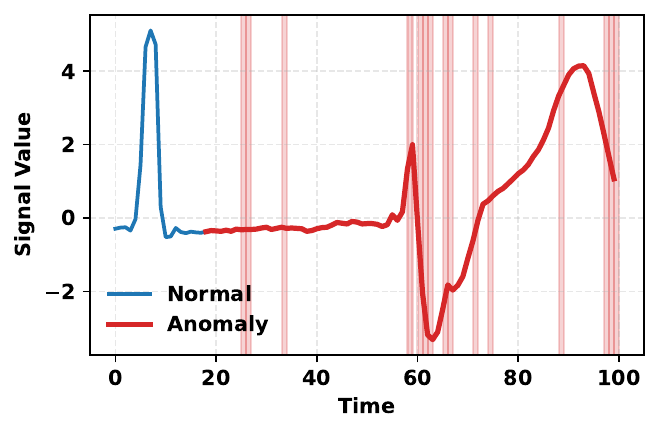}
        \caption{LTDB}
        \label{fig:subfig1}
    \end{subfigure}
    \hfill
    \begin{subfigure}[b]{0.48\linewidth}
        \centering
        \includegraphics[width=\linewidth]{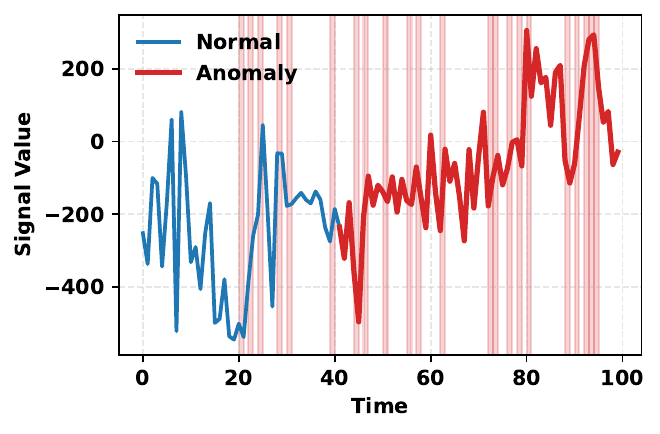}
        \caption{OPP}
        \label{fig:subfig2}
    \end{subfigure}
    
   \caption{Qualitative examples of generator masks on selected samples. Mask values are thresholded and highlighted in red.}
    \label{fig:two_figs}
\end{figure}

A central component of ARTA is the adversarial generator. While the generator's masks are visually interpretable, it is important to clarify what they explain. Although the generator produces temporal masks, these masks should be interpreted as explanations of the detector’s sensitivity rather than precise localizations of ground--truth anomalies. The generator is trained adversarially to identify worst--case perturbations that maximize the detector’s anomaly response, which need not coincide with the semantic or annotated extent of anomalies. As a result, the learned masks highlight regions that are most influential for the detector’s decision, rather than regions that are intrinsically anomalous. Results of Table~\ref{tab:ablation_comp} empirically support this distinction by showing that including the generator output in anomaly scoring leads to substantial performance degradation compared to the standard detector score, confirming that the adversarial masks are not suitable as standalone anomaly detectors or localization signals.
We therefore use the generator output exclusively for post--hoc interpretability analysis. Figure~\ref{fig:two_figs} presents qualitative examples illustrating the relationship between the input time--series and the corresponding masks. As seen in this figure, the masks consistently focus on sharp transitions or localized temporal regions that strongly influence the detector, even when these regions cover only a subset of the annotated anomaly window. This behavior is consistent with the adversarial objective, which emphasizes sensitivity rather than semantic completeness.
These results indicate that the adversarial masks learned by ARTA provide model--faithful explanations, offering insight into where and why the detector reacts strongly, without being conflated with the scoring mechanisms.

\section{Conclusion}

We introduced ARTA, an adversarial--robust framework for TSAD that jointly trains a detector and a sparsity--constrained mask generator in an adversarial formulation. By explicitly exposing the detector to worst--case, structured temporal perturbations during training, ARTA promotes reliance on stable, task--relevant temporal patterns rather than brittle point--wise cues.
Extensive evaluations on state--of--the--art benchmarks show that ARTA consistently outperforms state--of--the--art baselines and exhibits significantly improved robustness to input noise. Ablation studies further confirm the complementary roles of adversarial interaction, sparsity regularization, and mask--guided perturbations in driving both performance and robustness gains.
We hope this work serves as a benchmark for the next generation of TSAD methods, encouraging the field to move beyond unreliable evaluation metrics and to place greater emphasis on robustness in algorithm design.

\begin{figure*}[!t]
    \centering
    \Description{A figure showing robustness evaluation of anomaly detection methods under additive Gaussian noise with varying signal--to--noise ratios (SNR) across three representative datasets, averaged over five runs.}
    \begin{subfigure}{0.32\linewidth}
        \centering
        \includegraphics[width=\linewidth]{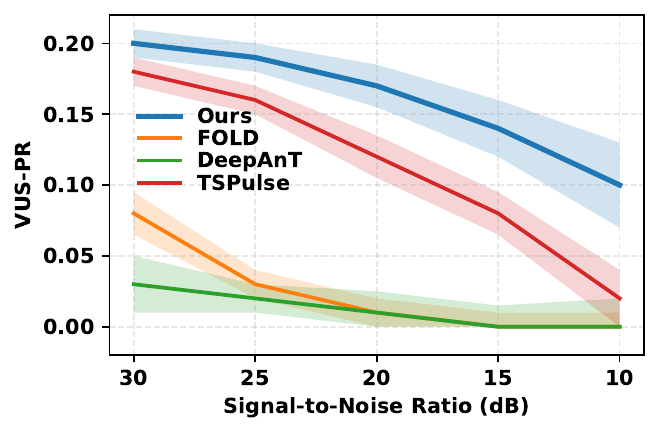}
        \caption{GECCO}
        \label{fig:robust_a}
    \end{subfigure}
    \hfill
    \begin{subfigure}{0.32\linewidth}
        \centering
        \includegraphics[width=\linewidth]{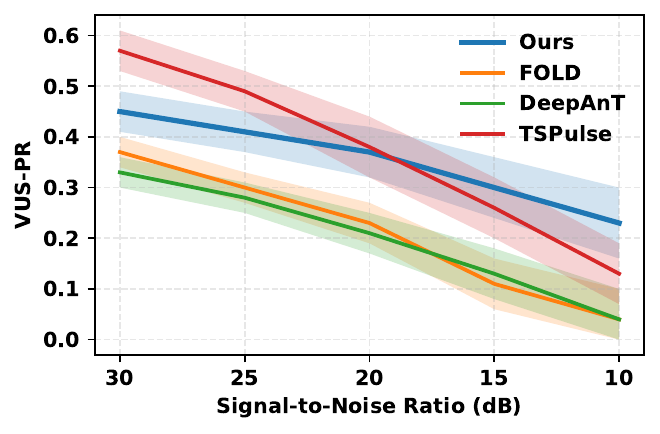}
        \caption{LTDB}
        \label{fig:robust_b}
    \end{subfigure}
    \hfill
    \begin{subfigure}{0.32\linewidth}
        \centering
        \includegraphics[width=\linewidth]{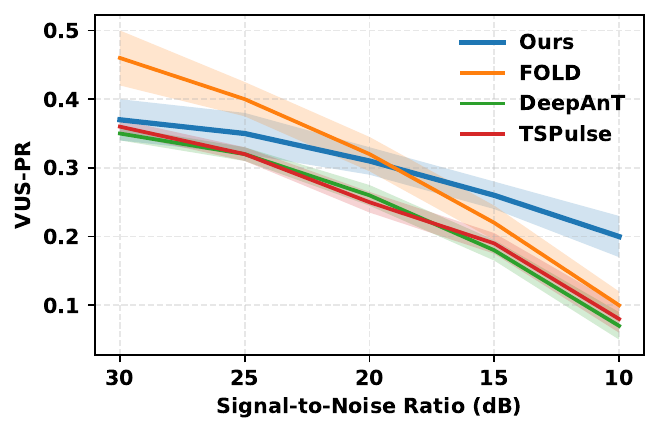}
        \caption{SMD}
        \label{fig:robust_c}
    \end{subfigure}

    \caption{Robustness evaluation of anomaly detection methods under additive Gaussian noise with varying signal--to--noise ratios (SNR) across three representative datasets, averaged over five runs. \b{Our method} consistently demonstrates slower performance degradation.}
    \label{fig:robustness}
\end{figure*}

\begin{table*}[!t]
    \centering
    \begin{tabular}{cccccccccc}
\hline
Hyperparameter & $0.0001$ & $0.001$ & $0.01$ & $0.1$ & $1$ & $10$ & $100$ & $1000$ &$10000$\\
\hline
Sparsity $\lambda$ ($\gamma = 0.1$) & 0.31 & 0.34 & \textbf{0.40} & 0.39 & 0.33 & 0.29 & 0.27 & 0.26 & 0.24\\
Robustness $\gamma$ ($\lambda = 0.01$) & 0.23 & 0.30 & 0.37 & \textbf{0.40}&0.38&0.36&0.30&0.28&0.27\\
\hline

    \end{tabular}
    \caption{Average performance across all datasets, computed as the mean VUS--PR, for different settings of the hyperparameters $\lambda$ and $\gamma$.}
    \label{tab:sensitivhp}
\end{table*}

\begin{acks}
To Robert, for the bagels and explaining CMYK and color spaces.
\end{acks}


\section*{Appendix}
\appendix

\section{Robustness to Gaussian Noise}
\label{app:nongauss}

In this section, we simulate noisy or adversarial--like conditions by applying additive Gaussian noise to the input time series:
\[
\tilde{X} = X + \epsilon, \quad \epsilon \sim \mathcal{N}(0, \sigma^2),
\]
where the noise variance $\sigma^2$ is selected to achieve a target signal--to--noise ratio (SNR). We systematically vary the SNR from high (mild noise) to low (severe noise), enabling a controlled evaluation of performance degradation under increasing corruption. Additive Gaussian noise is a widely used perturbation model that realistically reflects sensor noise while allowing reproducible control over noise intensity.

Anomaly detection performance under noise is evaluated using the VUS--PR metric. For each dataset, we record the VUS--PR score at each SNR level, producing a degradation curve that reflects robustness to perturbations. Figure~\ref{fig:robustness} presents these curves for four representative datasets, where the x--axis denotes SNR and the y--axis reports detection performance. Across all datasets, ARTA exhibits a noticeably slower degradation in performance compared to competing methods, indicating superior robustness to input noise. This is in line with the robustness trend that we observed for more challenging noise types in the main manuscript.

\section{Hyperparameter Sensitivity}

In this section, we investigate the sensitivity of model performance to the two hyperparameters introduced by our framework: $\gamma$, which controls the emphasis on minimizing the anomaly score for perturbed inputs, and $\lambda$, which regulates the strength of the sparsity constraint imposed on the generator.

To this end, we vary each hyperparameter independently and record the resulting performance. The results are reported in Table~\ref{tab:sensitivhp}. For the robustness hyperparameter $\gamma$, setting it to a small value leads to a significant performance degradation. This setting effectively ignores adversarial perturbations in the detector loss, which contradicts the core premise of our method. Conversely, excessively large values of $\gamma$ also harm performance, albeit to a lesser extent, since overemphasizing perturbed samples prevents the model from adequately learning the structure of clean, non--perturbed signals.

A similar trend is observed for $\lambda$. When $\lambda$ is set to a very small value, little sparsity constraint is imposed on the mask, allowing the generator to trivially corrupt the entire signal. As a result, the detector is trained on inputs that contain little to no meaningful information, leading to degraded performance. At the opposite extreme, excessively large values of $\lambda$ enforce overly sparse masks, severely restricting the generator’s flexibility in selecting timestamps to perturb. In this case, the resulting perturbations become easy for the detector to identify, preventing it from developing robustness against worst--case, localized perturbations.

\section{Additional Evaluation Metrics}
\label{app:metrics}

In Table~\ref{tab:metrics}, we report additional threshold--independent metrics that are standard in time--series anomaly detection, namely AUC--PR, AUC--ROC, VUS--ROC, and standard F1.

\begin{table}[!t]
\footnotesize
    \centering
    \begin{tabular}{ccccc}
\hline
Method & F1 ($\uparrow$) & AUC--PR ($\uparrow$) & AUC--ROC ($\uparrow$) & VUS--ROC ($\uparrow$) \\
\hline
PCA & 0.465 &0.383 & 0.798 & 0781 \\
OmniAnomaly &0.469&0.389 & 0.798 &0.786 \\
TSPulse (ZS) &0.358&0.307 & 0.698 & 0.726 \\
TSPulse (FT) &0.406& 0.347 & 0.755 & 0.769 \\
FOLD &0.432& 0.397 & 0.782 & 0.772 \\
\textbf{ARTA} &\textbf{0.507}&\textbf{0.414}&\textbf{0.845} & \textbf{0.860}\\
\hline
    \end{tabular}
    \caption{Comparison of F1, AUC--PR, AUC--ROC, and VUS--ROC between our proposed method and top--5 baselines. The reported values are averaged over all datasets.}
    \label{tab:metrics}
\end{table}

\section*{Gen AI Usage Disclosure}

We used generative AI tools exclusively to assist with English language proofreading, grammar correction, and improving the clarity and readability of the manuscript. All scientific ideas, experimental design, analyses, results, and conclusions were developed and verified solely by the authors. The AI assistance did not contribute to any technical content of the work.

\bibliographystyle{ACM-Reference-Format}
\bibliography{refs}

\end{document}